\documentclass[twoside]{article}

\usepackage[accepted]{aistats2021}

\usepackage{microtype}
\usepackage{graphicx}
\usepackage{subfigure}
\usepackage{booktabs} 

\usepackage{amsfonts,amsmath,amssymb,amsthm,graphicx,xfrac}
\usepackage{hyperref}

\usepackage[linesnumbered,boxed,norelsize]{algorithm2e}
\usepackage{algorithmic}

\newcommand{\citet}{\cite}
\newcommand{\citep}{\cite}

\newtheorem{theorem}{Theorem}
\newtheorem{lemma}[theorem]{Lemma}
\newtheorem{corollary}[theorem]{Corollary}

\newcommand{\diam}{\mathrm{diam}}
\newcommand{\ud}{\mathrm{d}}
\newcommand{\calA}{\mathcal{A}}
\newcommand{\calT}{\mathcal{T}}
\newcommand{\poly}{\mathrm{poly}}

\newcommand{\vct}{\boldsymbol }

\newcommand{\E}{\mathbb E}

\renewcommand{\tilde}{\widetilde}
\renewcommand{\hat}{\widehat}

\begin{document}

\twocolumn[

\aistatstitle{Tight Regret Bounds for Infinite-armed Linear Contextual Bandits}

\aistatsauthor{Yingkai Li \And Yining Wang }
\aistatsaddress{ Northwestern University \\ {\tt yingkai.li@u.northwestern.edu} \And University of Florida \\ {\tt yining.wang@warrington.ufl.edu}} 
\aistatsauthor{Xi Chen \And Yuan Zhou}
\aistatsaddress{ New York University \\ {\tt xc13@stern.nyu.edu} \And University of Illinois at Urbana-Champaign \\ {\tt yuanz@illinois.edu} } ]

\begin{abstract}
Linear contextual bandit is an important class of sequential decision making problems with a wide range of applications to recommender systems, online advertising, healthcare, and many other machine learning related tasks.
While there is a lot of prior research, tight regret bounds of linear contextual bandit with infinite action sets remain open.
In this paper, we address this open problem by considering the linear contextual bandit with (changing) infinite action sets.
We prove a regret upper bound on the order of $O(\sqrt{d^2T\log T})\times \poly(\log\log T)$
where $d$ is the domain dimension and $T$ is the time horizon.
Our upper bound matches the previous lower bound of $\Omega(\sqrt{d^2 T\log T})$ in \citep{li2019near}
up to \emph{iterated} logarithmic terms.
\end{abstract}

\fancyhead[CO]{\small\bfseries Yingkai Li, Yining Wang, Xi Chen, Yuan Zhou}

\section{Introduction}

Linear contextual bandit is an important class of sequential decision making problems with an extensive history of research
in both machine learning and operations research \citep{abbasi2011improved,chu2011contextual,auer2002using,rusmevichientong2010linearly,dani2008stochastic,li2019near}.
In the linear contextual bandit problem, a player makes sequential decisions over $T$ time periods.
At each time period $t$, an \emph{action set} $D_t\subseteq\mathbb R^d$ is provided;
the player selects an \emph{action} $x_t\in D_t$, and subsequently receives a \emph{reward} $r_t$ parameterized as
$$
r_t = \langle x_t, \theta\rangle + \xi_t,
$$
where $\theta\in\mathbb R^d$, $\|\theta\|_2\leq 1$ is a fixed but unknown parameter vector, and $\{\xi_t\}$ are independent centered sub-Gaussian noise variables
with the variance proxy 1.
The performance is evaluated by the cumulative \emph{regret}, defined as
$$
R_T := \sum_{t=1}^T \sup_{x\in D_t}\langle x, \theta\rangle - \langle x_t, \theta\rangle.
$$

The objective of this paper is to design an algorithm that achieves the \emph{optimal} expected regret under the worst case,
when the action sets $\{D_t\}$ are \emph{infinite} (i.e., $|D_t|=\infty$).
In the next sections, we give a rigorous definitions of policy and action domains studied in this paper. 
We also discuss (informally)
our main results, and compare them with existing results in the literature.

\subsection{Definition of policy and action domains}

Suppose that there are $T$ time periods and the problem dimension is $d$.
A policy $\pi$ can be represented as $\pi=(\phi_1,\phi_2,\cdots,\phi_T)$
where $\phi_t: (x_1,y_1,\cdots,x_{t-1},y_{t-1}, D_t)\mapsto x_t\in D_t$ is a randomized function that maps 
the data collected from prior episodes $\{1,2,\cdots,t-1\}$ to an action $x_t\in D_t$ to be selected at time period $t$.
Note that future feasible sets $D_{t+1},D_{t+2}$ are \emph{not} revealed to the policy $\pi$ when it is making an action decision
at time $t$.

Let $\mathcal S_d := \left\{S: S\text{ is closed}, S\subseteq\{x\in\mathbb R^d: \|x\|_2\leq 1\} \right\}$ be the set of all closed subsets of the unit $d$-dimensional
$\ell_2$ ball.
The domains $D_1,\cdots,D_T\in\mathcal S_d$ are chosen arbitrarily, \emph{before} any policy $\pi$ is executed. We remark that this setting is known in the literature as the ``oblivious'' setting.


\subsection{Existing work and our results}

\begin{table*}[!t]
\centering
\caption{Summary of results. Both $\theta$ and $\{D_t\}$ belong to the unit ball $\{x\in\mathbb R^d:\|x\|_2\leq 1\}$, and
$|D_t|=\infty$ for all $t$. Upper and lower bounds are for $\mathbb E[R_T]$ under the worst case.
$O(\cdot)$ and $\Omega(\cdot)$ notations hide universal constants only, and $\poly(\log\log T)$ means $(\log\log T)^{O(1)}$.}
\vskip 0.15in
\scalebox{0.85}{
\begin{tabular}{lcccc}
\toprule
& \citet{dani2008stochastic}& \citet{abbasi2011improved}& \citet{li2019near}& \textbf{this paper}\\
\midrule
Upper bound& $O(\sqrt{d^2 T\log^3 T})$& $O(\sqrt{d^2 T\log^2 T})$&N/A&   \shortstack{$O(\sqrt{d^2 T\log T})$$\times \poly(\log\log T)$} \\
Lower bound& $\Omega(\sqrt{d^2 T})$&N/A& $\Omega(\sqrt{d^2 T\log T})$& N/A \\
\bottomrule
\end{tabular}
}
\label{tab:summary}
\end{table*}

A summary of our results as well as existing results are given in Table \ref{tab:summary}.
The regularity conditions that $\|\theta\|_2\leq 1$ and $D_t\subseteq\{x\in\mathbb R^d:\|x\|_2\leq 1\}$ are imposed,
so that $|\mathbb E[r_t]|=|\langle x_t,\theta\rangle| \leq 1$ holds for all $x_t\in D_t$.
Additionally, as suggested by the title, we consider the \emph{infinite-armed} case in which $|D_t|=\infty$ for all $t$.
We also impose the regularity condition that the action sets $D_t$ are \emph{closed}, so that the supremum over the sets can always be achieved by an action. 

\citet{dani2008stochastic} derived an algorithm based on confidence balls of prediction errors of $\theta$,
achieving a worst-case expected regret of $O(\sqrt{d^2T\log^3 T})$.
\citet{abbasi2011improved} further improved the analysis and obtained $O(\sqrt{d^2 T\log^2 T})$ regret.
On the lower bound side, \citet{dani2008stochastic} proved a regret lower bound of $\Omega(\sqrt{d^2 T})$ for all policies,
which was later improved to $\Omega(\sqrt{d^2 T\log T})$ by \citet{li2019near} as a direct corollary of regret lower bounds for finite-armed
linear contextual bandits.
While \citet{li2019near} derived matching upper bounds for the finite-armed case, their results and techniques cannot be directly applied to the infinite-armed case
even if computational issues are disregarded, as covering nets of $\{D_t\}$ up to $1/\poly(T)$ accuracy would incur additional logarithmic terms in $T$.

In this paper, we prove the following main result:
\begin{theorem}[Informal]
There is a policy whose worst-case expected regret is asymptotically upper bounded by 
$O(\sqrt{d^2 T\log T})\times \poly(\log\log T)$.
\label{thm:informal}
\end{theorem}

Comparing with the lower bound $\Omega(\sqrt{d^2 T\log T})$, the upper bound in Theorem \ref{thm:informal} is tight up to iterated logarithmic terms.
Our results thus close the $O(\sqrt{\log T})$ gap between the upper bound (in \citet{abbasi2011improved}) and the lower bound in infinite-armed linear contextual bandit. 
{ In addition, the idea behind our varying confidence level (VCL) UCB algorithm and  a number of technical tools developed in the proof might also be useful for other contextual bandit problems.}

\subsection{Proof techniques}


\paragraph{Sharp tail bounds of self-normalized empirical processes.}
Due to the inherent statistical dependency between the chosen actions $\{a_t\}$ and noise variables $\{\xi_t\}$, 
the estimation error of $\theta$ at each time step cannot be analyzed using standard closed-forms of linear regression estimators.
The work of \citet{abbasi2011improved} pioneered the use of \emph{self-normalized empirical processes} to understand the estimation and prediction errors at each time step.

In this paper, we make use of sharp tail bounds on the supremum of self-normalized empirical processes in high-dimensional probability (Lemma \ref{lem:uniform-tail-bound}).
By exploiting such tail bounds we have a much more refined control of failure probabilities at each time step, which lays the foundation of our improved regret analysis.

\paragraph{Varying confidence levels  in UCB-type algorithms.}
Most existing methods on linear contextual bandit can be categorized as \emph{Upper-Confidence-Bound (UCB)} or \emph{Optimism-in-Face-of-Uncertainty (OFU)} type algorithms,
which build confidence bands/balls around unknown  parameters at each time step and then pick actions in the most optimistic way.

While most existing algorithms set constant confidence levels (corresponding to failure probabilities at each time),
in this paper we consider \emph{varying confidence levels} (VCL), with higher failure probabilities towards the end of the time horizon $T$. 
The intuition is that later fails would incur much less regret.
Similar ideas were also employed in previous works \citep{audibert2009minimax,li2019near,wang2018mnl}
to improve regret guarantees in bandit problems.

\section{Algorithm design and main results}

\begin{algorithm*}[t]
\caption{The VCL-SupLinUCB algorithm}
\label{alg:suplinucb}
\begin{algorithmic}[1]
\STATE \textbf{Input}:
$\zeta_0=\lceil\log_2(\sqrt{T/d} /\delta)\rceil$, 
Time horizon $T$, confidence parameter $\delta$, domain dimension $d$, universal constant $C \geq 1$\;
\STATE \textbf{Initialization}: $\mathcal X_{\zeta,0}=\emptyset$,
$\Lambda_{\zeta,0} = I_{d\times d}$, $\lambda_{\zeta,0} = \vec{0}_d$, $\hat\theta_{\zeta,0}=\vec{0}_d$ 
for $\zeta \leq \zeta_0$;
\FOR{$t=1,2,\cdots,T$}
	 \STATE Observe $D_t$, and set $\zeta=0$ and $\mathcal N_{\zeta,t}=D_t$;
	\WHILE{a choice $x_t$ has yet to be made}
	 \STATE Compute $\hat\theta_{\zeta,t}=\Lambda_{\zeta,t-1}^{-1}\lambda_{\zeta,t-1}$, and  for every $x \in\mathcal N_{\zeta,t}$, compute  $\omega_{\zeta,t}^x = \sqrt{x^\top\Lambda_{\zeta,t-1}^{-1}x}$,
$\alpha_{\zeta,t}^x = \sqrt{\max\{1, \ln[(T\ln^4 T \ln^2 (1/\delta))(\omega_{\zeta,t}^x)^2/(d\delta^2)]\}}$ 
and 
$\varpi_{\zeta,t}^x = C \cdot \sqrt{d} \cdot \alpha_{\zeta,t}^x\omega_{\zeta,t}^x$;
\IF{$\zeta = \zeta_0$} 
	 	  \STATE Find $x_t\in\mathcal N_{\zeta,t}$ that maximizes $\min\{1,x_{it}^\top\hat\theta_{\zeta,t} +\varpi_{\zeta,t}^x\}$ and set $\zeta_t=\zeta$;
	 \ELSIF{$\varpi_{\zeta,t}^x\leq 2^{-\zeta}$ for all $x\in\mathcal N_{\zeta,t}$}
	 	 \STATE Update $\mathcal N_{\zeta+1,t} =\mathcal N_{\zeta,t} \cap \{x : x^\top\hat\theta_{\zeta,t}\geq\max_{y\in\mathcal N_{\zeta,t}}y^\top\hat\theta_{\zeta,t}-2^{1-\zeta}\}$, $\zeta\gets\zeta + 1$;
	 \ELSE
	 	\STATE Select any $x_t\in\mathcal N_{\zeta+1,t}$ such that $\varpi_{\zeta,t}^x\geq 2^{-\zeta}$, and set $\zeta_t=\zeta$;
	 \ENDIF
	\ENDWHILE
	\STATE Select action $x_t$ and observe feedback $r_t=x_{t}^\top\theta+\xi_t$;
	 \STATE Update: $\mathcal X_{\zeta_t,t}=\mathcal X_{\zeta_t,t-1}\cup\{x_t\}$, $\Lambda_{\zeta_t,t}=\Lambda_{\zeta_t,t-1}+x_{t}x_{t}^\top$, $\lambda_{\zeta_t,t}=\lambda_{\zeta_t,t-1} + r_tx_{t}$,
	 and $\mathcal X_{\zeta',t}=\mathcal X_{\zeta',t-1}$, $\Lambda_{\zeta',t}=\Lambda_{\zeta',t-1}$, $\lambda_{\zeta',t}=\lambda_{\zeta',t-1}$ for any $\zeta'\neq\zeta_t$;
\ENDFOR
\end{algorithmic}
\end{algorithm*}

Algorithm \ref{alg:suplinucb}, named VCL-SupLinUCB, is the main algorithm of this paper which combines the varying confidence levels (VCL) design
with the existing SupLinUCB algorithm \citep{auer2002using,chu2011contextual}.
The basic idea of Algorithm \ref{alg:suplinucb} is to classify the time periods into different layers such that 
the chosen context is statistically independent with the noise in the reward distribution. 
Then the algorithm estimates $\hat\theta_{\zeta,t}$ for each layer $\zeta$ and eventually selects the arm with largest upper confidence bound according to the rules specified in Lines 7-12 in Algorithm \ref{alg:suplinucb}. 
Those ideas are similar to the previous SupLinUCB algorithm, 
and the major difference between Algorithm \ref{alg:suplinucb} and previous approaches is the \emph{varying} confidence levels
(reflected by the inclusion of $\omega_{x,t}$ in $\alpha_{x,t}$),
which allows for sharper regret bounds.

The following theorem is the main result of this paper:


\begin{theorem}\label{thm:main}\label{thm:main-upper}
Suppose the universal constant $C>0$ in the input of Algorithm \ref{alg:suplinucb} is sufficiently large. Then there exists constants $C_1, C_2 > 0$ that only depend on $C$ such that for all $\|\theta\|_2\leq 1$ and $\{D_t\subseteq\{x\in\mathbb R^d:\|x\|_2\leq 1\} \}$, the regret 
 $R_T$ satisfies the following inequality {for any $\delta\in(0,1)$},
\begin{align*}
&\E\left[\max\left\{R_T - C_1  d \sqrt{ T \log T \log (1/\delta)} \cdot \log \log (T/\delta), 0\right\}\right] \\
&\leq C_2 \delta d\sqrt{T}.
\end{align*}
\end{theorem}

Theorem \ref{thm:main-upper} implies the following two statements. 
In the following, $\lesssim$ means that the constants in the inequality are omitted. 
\begin{enumerate}
\item[i.] The expected regret $\E[R_T] \lesssim  d \sqrt{ T \log T \log (1/\delta)} \cdot \log \log (T/\delta)$. In particular, if we take $\delta = \Omega(1)$, we have that $\E[R_T] \lesssim d \sqrt{ T \log T} \cdot \log \log T$.
\item[ii.] With probability at least $1 - \delta$, it holds that $R_T \lesssim  d \sqrt{ T \log T \log (1/\delta)} \cdot \log \log (T/\delta)$.
This is because, by Markov's inequality, $\Pr[R_T-C_1d\sqrt{T\log T\log(1/\delta)}\cdot\log\log(T/\delta) > C_2d\sqrt{T}] \leq \mathbb E[\max\{R_T-C_1d\sqrt{T\log T\log(1/\delta)}\cdot\log\log(T/\delta),0\}]/(C_2d\sqrt{T}) \leq \delta$.
\end{enumerate}
While neither of the two statements implies each other, we note that, if iterated logarithmic factor is left out, 
statement ii) is stronger than the high probability bound proved by \cite{abbasi2011improved}, where the regret is at most $O(d \sqrt{T \log T \log (T/\delta)})$ with probability at least $1-\delta$.

The proof of Theorem \ref{thm:main-upper} is stated in the next section.

\section{Proof of Theorem \ref{thm:main-upper}}\label{sec:main-proof}

\subsection{Uniform confidence region for $\hat\theta_{\zeta, t}$}

We first present a lemma that upper bounds the errors $|\langle x,\hat\theta_{\zeta, t}-\theta\rangle|$ with high probability.

\begin{lemma}\label{lem:uniform-tail-bound}
For any $t\in[T]$, any layer $\zeta \in \{0, 1, 2, \dots, \zeta_0\}$, and any $\gamma\in(0,1/2]$, with probability $1-\gamma$ it holds that
\[
\sup_{x\in\mathbb R^d}(\omega^x_{\zeta,t})^{-1}{\big|x^\top(\hat\theta_{\zeta, t}-\theta)\big|} \lesssim \sqrt{d} + \sqrt{\ln(1/\gamma)} .
\]
\label{lem:empirical-process}
\end{lemma}

The proof of Lemma \ref{lem:empirical-process} can be roughly divided into three steps.
First, the closed-form expression of Ridge regression to express $\hat\theta_{\zeta, t}$ in terms of $\theta$ and $\xi$.
At the second step, a self-normalized empirical process is derived by manipulating and normalizing the expression derived in the first step.
Finally, sharp tail bounds of sub-Gaussian processes are invoked to prove Lemma \ref{lem:empirical-process}.

\begin{proof}[Proof of Lemma \ref{lem:empirical-process}]
Let 
\[
\mathcal{T}_{\zeta, t-1} := \{\tau: \tau \leq t - 1 \text{~and~} \zeta_\tau = \zeta\}
\]
and let 
\[
n_{\zeta, t-1} := |\mathcal{T}_{\zeta, t-1}|.
\]
Note that we also have $n_{\zeta, t-1} = |\mathcal{X}_{\zeta, t-1}|$. 

Let $X_{\zeta, t-1}$ be a $n_{\zeta, t-1} \times d$ matrix constructed by stacking all $x \in \mathcal{X}_{\zeta, t-1}$ together, i.e., 
\[
\Lambda_{\zeta, t-1}=X_{\zeta, t-1}^\top X_{\zeta, t-1} + I.
\]
 Let $\Xi_{\zeta, t-1}$ be the $n_{\zeta, t-1}$-dimensional vector that contains all noises $\xi_\tau$ such that $\tau  \in \mathcal{T}_{\zeta, t-1}$. We also let 
 \[
 r_{\zeta, t-1} = X_{\zeta, t-1}\theta+\Xi_{\zeta, t-1}
 \]
  be the $n_{\zeta, t-1}$-dimensional vector by concatenating all rewards for time periods $\tau\in \mathcal{T}_{\zeta, t-1}$.

Define also $\|x\|_A := \sqrt{x^\top Ax}$ for $d$-dimensional vectors $x$ and $d\times d$ positive-semidefinite matrices $A$. Then
\begin{align*}
\hat\theta_{\zeta, t} &= (X_{\zeta, t-1}^\top X_{\zeta, t-1}+I)^{-1}X_{\zeta, t-1}^\top(X_{\zeta, t-1}\theta + \Xi_{\zeta, t-1}) \\
&= (I-\Lambda_{\zeta, t-1}^{-1})\theta + \Lambda_{\zeta, t-1}^{-1}X_{\zeta, t-1}^\top\Xi_{\zeta, t-1}.
\end{align*}
Subtracting one $\theta$ and left multiplying with $(\hat\theta_{\zeta, t}-\theta)^\top\Lambda_{\zeta, t-1}$ on both sides of the above identity, we obtain
\begin{equation}
\|\hat\theta_{\zeta, t}-\theta\|_{\Lambda_{\zeta, t-1}}^2 = -(\hat\theta_{\zeta, t}-\theta)^\top\theta + (\hat\theta_{\zeta, t}-\theta)^\top X_{\zeta, t-1}^\top\Xi_{\zeta, t-1}.
\label{eq:basic-ineq1}
\end{equation}
Note that 
\[
|(\hat\theta_{\zeta, t}-\theta)^\top\theta|\leq \|\theta\|_2\|\hat\theta_{\zeta, t}-\theta\|_2 \leq \|\hat\theta_{\zeta, t}-\theta\|_{\Lambda_{\zeta, t-1}}
\]
because $\|\theta\|_2\leq 1$ and $\Lambda_{\zeta, t-1} \succeq I$. Dividing both sides of Eq.~(\ref{eq:basic-ineq1}) by $\|\hat\theta_{\zeta, t}-\theta\|_{\Lambda_{\zeta, t-1}}$, we have
\begin{align}
\|\hat\theta_{\zeta, t}-\theta\|_{\Lambda_{\zeta, t-1}} \leq 1 + \phi^\top X_{\zeta, t-1}^\top\Xi_{\zeta, t-1},\nonumber\\
\;\;\text{where}\;\; \phi = (\hat\theta_{\zeta, t}-\theta)/\|\hat\theta_{\zeta, t}-\theta\|_{\Lambda_{\zeta, t-1}}.
\label{eq:basic-ineq2}
\end{align}
It is easy to verify that $\phi$ satisfies $\|\phi\|_{\Lambda_{\zeta, t-1}}\leq 1$.
Consider linear transforms $\tilde x_\tau = \Lambda_{\zeta, t-1}^{-1/2}x_\tau$ for all $\tau\in \mathcal{T}_{\zeta, t-1}$ and $\tilde\phi=\Lambda_{\zeta, t-1}^{1/2}\phi$.
Then $\tilde\phi$ satisfies $\|\tilde\phi\|_2\leq 1$. Subsequently, Eq.~(\ref{eq:basic-ineq2}) can be re-formulated as
\begin{equation}
\|\hat\theta_{\zeta, t}-\theta\|_{\Lambda_{\zeta, t-1}} \leq 1  + \sup_{\|\tilde\phi\|_2\leq 1} G_{\tilde\phi},
\label{eq:basic-ineq3}
\end{equation}
where $ G_{\tilde\phi} = \sum_{\tau \in \mathcal{T}_{\zeta, t-1}}\xi_\tau\langle \tilde x_\tau,\tilde\phi\rangle.$

We next show that $G_\cdot$ is a sub-Gaussian process with respect to  $\|\cdot\|_2$. Since $\{\xi_\tau\}_{\tau\in\mathcal T_{\zeta,t-1}}$ and $\{x_{\tau}\}_{\tau\in\mathcal T_{\zeta,t-1}}$ are statistically independent \citep{chu2011contextual,auer2002using}, we have that $\{\xi_\tau\}_{\tau\in\mathcal T_{\zeta,t-1}}$ and $\{\tilde x_\tau\}_{\tau\in\mathcal T_{\zeta,t-1}}$ are statistically independent. 
Therefore, for any $\phi,\phi'$, $G_\phi-G_{\phi'}=\sum_{\tau\in \mathcal{T}_{\zeta, t-1}}\xi_\tau\langle\tilde x_\tau,\phi-\phi'\rangle$ is a centered sub-Gaussian random variable with variance proxy
\begin{align*}
&~~\sum_{\tau\in \mathcal{T}_{\zeta, t-1}}|\langle \tilde x_\tau,\phi-\phi'\rangle|^2 \\
&= (\phi-\phi')^\top\Big(\sum_{\tau\in\mathcal{T}_{\zeta, t-1}}\tilde x_\tau\tilde x_\tau^\top\Big)(\phi-\phi')\\
&= (\phi-\phi')^\top\Lambda_{\zeta, t-1}^{-1/2}\Big(\sum_{\tau\in\mathcal{T}_{\zeta, t-1}}x_\tau x_\tau^\top\Big)\Lambda_{\zeta, t-1}^{-1/2}(\phi-\phi')\\
&\leq \|\phi-\phi'\|_2^2 .
\end{align*}
Subsequently, invoking Lemma \ref{lem:gp-tail}, we have with probability $1-\gamma$ that
\begin{align*}
&\|\hat\theta_{\zeta, t}-\theta\|_{\Lambda_{\zeta, t-1}}\\
\lesssim& 1+\int_0^{\infty} \sqrt{\ln N(\{x\in\mathbb R^d:\|x\|_2\leq 1\}; \|\cdot\|_2,\epsilon)}\ud\epsilon\\
&\qquad + \sqrt{\ln(1/\gamma)}\\
\lesssim& 1+ \int_0^2\sqrt{d\ln(1/\epsilon)}\ud\epsilon + \sqrt{\ln(1/\gamma)} \lesssim \sqrt{d}+\sqrt{\ln(1/\gamma)}.
\end{align*}

Finally, Lemma \ref{lem:empirical-process} is proved by the Cauchy-Schwarz inequality:
\begin{align*}
\big|x^\top(\hat\theta_{\zeta, t}-\theta)\big| &\leq \|x\|_{\Lambda_{\zeta, t-1}^{-1}}\|\hat\theta_{\zeta, t}-\theta\|_{\Lambda_{\zeta, t-1}}\\
&\leq \omega^x_{\zeta,t}\|\hat\theta_{\zeta, t}-\theta\|_{\Lambda_{\zeta, t-1}}, \;\;\;\;\forall x\in\mathbb R^d,
\end{align*}
which is to be demonstrated.
\end{proof}

In this paper, we only use the following weaker version of Lemma~\ref{lem:uniform-tail-bound}. 

\begin{corollary}\label{cor:uniform-tail-bound}
For any $t\in[T]$, any layer $\zeta \in \{0, 1, 2, \dots, \zeta_0\}$, and any $\gamma\in(0,1/2]$, with probability $1-\gamma$ it holds that
\[
\sup_{x\in\mathbb R^d}(\omega^x_{\zeta,t})^{-1}{\big|x^\top(\hat\theta_{\zeta, t}-\theta)\big|} \lesssim \sqrt{d \cdot \ln(1/\gamma)} .
\]
\end{corollary}

\subsection{Regret upper bound at a single time step}

For each time $t$, we first bound the expected error of the estimation of any arm that lies out of its confidence band 
using the sharp bound we obtained in Corollary \ref{cor:uniform-tail-bound}.


\begin{lemma}\label{lem:reward-error-bound-in-each-layer}
There exists a sufficiently large universal constant $C > 0$ such that for each layer $\zeta \in \{0, 1, 2, \dots, \zeta_0\}$, 
for each time $t \in [T]$, and for any $\delta \in (0, 1/2]$,  it holds that
\begin{align}\label{eq:reward-error-bound-in-each-layer}
\E\left[\max_{x \in D_t} \left\{\mathbf{1}\left[ \left| x^\top (\hat{\theta}_{\zeta, t} - \theta)\right| 
\geq  C \sqrt{d } \cdot  \alpha_{\zeta, t}^{x}  \cdot 
\omega_{\zeta, t}^{x}  \right]
\right.\right.\nonumber\\
\left.\left.
\cdot \left| x^\top (\hat{\theta}_{\zeta, t} - \theta)\right| \right\} \right] 
\lesssim \delta d / \sqrt{T \ln^2 T}.
\end{align}
\end{lemma}

To prove Lemma~\ref{lem:reward-error-bound-in-each-layer}, we adopt a novel argument that partitions the action set according to the geometric scale of the confidence levels of the actions. 
Using Corollary~\ref{cor:uniform-tail-bound}, for each partition, we derive a uniform error bound for the expected reward of the actions in the partition with empirical estimate $\hat{\theta}_{\zeta, t}$. 
Since we have no control on the index of the partition that the maximizer $x^*$ belongs to, we finally employ a union bound argument to combine the error bounds for every partition and complete the proof. 


\begin{proof}[Proof of Lemma~\ref{lem:reward-error-bound-in-each-layer}]
For each layer $\zeta$ 
and for each time $t$,
consider a partition of $D_t$, namely 
$\{\calA_{\zeta, t}^{\kappa}\}_{\kappa \in \{1, 2, 3, \dots, K\}}$, 
where $K = \lceil \log_2 (T^2 /\delta^2)\rceil + 1$, and we define
\begin{align*}
\calA_{\zeta, t}^{\kappa} =  \left\{
\begin{array}{ll}
\{x\in D_t : \omega_{\zeta, t}^x \in (2^{-\kappa}, 2^{-\kappa + 1}]\} & \text{when~} \kappa<K\\
\{x\in D_t : \omega_{\zeta, t}^x \in (0, 2^{-\kappa + 1}]\}  & \text{when~} \kappa = K
 \end{array}
\right. 
\end{align*}
For each $\kappa$, we let 
\[
m_{\zeta, t}^\kappa = \sup_{i \in \calA_{\zeta, t}^{\kappa}} \left\{\left| x_{it}^\top (\hat{\theta}_{\zeta,t} - \theta)\right|\right\}
\]
be the maximum estimation error for the context vectors in $\calA_{\zeta, t}^{\kappa}$
and next we first provide bounds on $m_{\zeta, t}^\kappa$. 
By Corollary~\ref{cor:uniform-tail-bound}, there exists a universal constant $C$, 
such that for all $\beta \geq \sqrt{\ln 2}$, we have that 
\begin{align*}
&\Pr\left[ m_{\zeta, t}^\kappa \geq C \cdot 2^{-\kappa} \sqrt{d} \cdot \beta\right] \\
&= \Pr\left[ m_{\zeta, t}^\kappa \geq (C/2) \cdot 2^{-\kappa + 1} \sqrt{d} \cdot \beta\right] \\
& \leq \Pr\left[ \exists i \in \calA_{\zeta, t}^{\kappa}:   \left| x_{it}^\top (\hat{\theta}_{\zeta, t} - \theta)\right| 
\geq (C/2) \cdot  \omega_{\zeta, t}^i \sqrt{d} \cdot \beta\right] \\
&\leq e^{-\beta^2} .
\end{align*}

Now we let 
\[
\alpha_t^\kappa = \sqrt{\max\{1, {\ln [T \ln^4 T \ln^2(1/\delta) \cdot 2^{-2\kappa}/(d\delta^2)]}\}},
\]
 and use $\mathbf{1}[\cdot] $ to denote the indicator function. For each $\kappa$, it holds that
\begin{align}
&\E \left[ \mathbf{1} \left[ m_{\zeta, t}^\kappa \geq C \cdot 2^{-\kappa} 
\sqrt{d} \cdot \alpha_t^{\kappa} \right] \cdot m_{\zeta, t}^\kappa \right] \nonumber \\
\leq\ & \Pr\left[m_{\zeta, t}^\kappa \geq C \cdot 2^{-\kappa} \sqrt{d} \cdot \alpha_t^{\kappa} \right] 
\cdot \left( C \cdot 2^{-\kappa} \sqrt{d} \cdot \alpha_t^{\kappa} \right) \nonumber\\
&\quad + \int_{C \cdot 2^{-\kappa} \sqrt{d} \cdot \alpha_t^{\kappa} }^{+\infty} \Pr[m_{\zeta, t}^\kappa \geq z] dz  \nonumber \\
\leq\ &  \exp\left(- (\alpha_t^\kappa)^2\right)  \cdot C \cdot 2^{-\kappa} \sqrt{d} \cdot \alpha_t^{\kappa} \nonumber \\
& \quad + C \cdot 2^{-\kappa} \sqrt{d} \cdot 
\int_{\alpha_t^\kappa}^{\infty} e^{-\beta^2} d\beta  \nonumber \\
\lesssim\ & \exp\left(- (\alpha_t^\kappa)^2\right)  \cdot 2^{-\kappa} \sqrt{d} \cdot \alpha_t^{\kappa}. \label{eq:lem-reward-error-bound-1}
\end{align}

We now upper bound Eq.~\eqref{eq:lem-reward-error-bound-1} by considering the following two cases.

In the first case, when $\alpha_t^\kappa = 1$, 
we have that 
\[
T \ln^4 T \ln^2(1/\delta) \cdot 2^{-2\kappa}/(d\delta^2) \leq e,
\] 
which means that, 
\[
2^{-\kappa} \lesssim \sqrt{d \delta^2 / (T \ln^4 T \ln^2(1/\delta))}.
\] 
Therefore, Eq.~\eqref{eq:lem-reward-error-bound-1} is upper bounded by 
\[
e^{-1} \cdot 2^{-\kappa} \sqrt{d}  \lesssim \delta d  /\sqrt{T\ln^4 T \ln^2(1/\delta)}.
\]
 
In the second case, when $\alpha > 1$, we have that 
\[
T \ln^4 T \ln^2(1/\delta) \cdot 2^{-2\kappa}/(d\delta^2) = \exp((\alpha_{t}^{\kappa})^2)
\]
and therefore 
\[
2^{-\kappa} = \delta \sqrt{d/(T\ln^4 T \ln^2(1/\delta))} \cdot \exp((\alpha_{t}^{\kappa})^2/2).
\]
We can upper bound Eq.~\eqref{eq:lem-reward-error-bound-1} by 
\begin{align*}
\exp\left(-(\alpha_t^\kappa)^2 (1 - 1/2)\right) \cdot \delta \sqrt{d/(T\ln^4 T \ln^2(1/\delta))}
\cdot  \sqrt{d} \cdot \alpha_t^{\kappa}  \\
\lesssim \delta d/\sqrt{T\ln^4 T \ln^2(1/\delta)}.
\end{align*}
Summarizing the two cases, we have 
\begin{align}
\E& \left[ \mathbf{1} \left[ m_{\zeta, t}^\kappa 
\geq C \cdot 2^{-\kappa} \sqrt{d} \cdot \alpha_t^{\kappa} \right] 
 m_{\zeta, t}^\kappa \right] \nonumber\\
 &\qquad\qquad\qquad\lesssim \delta d /\sqrt{T \ln^4 T \ln^2(1/\delta)} . \label{eq:lem-reward-error-bound-2}
\end{align}

We now work with the Left-Hand Side of Eq.~\eqref{eq:reward-error-bound-in-each-layer}. Let $x^*$ be the maximizer in the LHS of Eq.~\eqref{eq:reward-error-bound-in-each-layer}, and let $\kappa^*$ be the index of the partition such that $x^* \in \calA_{\zeta, t}^{\kappa^*}$. We have 
\begin{align}
\E&\left[\mathbf{1}\left[ \left| (x^*)^\top (\hat{\theta}_{\zeta, t} - \theta)\right| \geq  C \sqrt{d} \cdot \alpha_{\zeta, t}^{x^*} \omega_{\zeta, t}^{x^*}  \right]
\times \left| (x^*)^\top (\hat{\theta}_{\zeta, t} - \theta)\right|\right]\nonumber\\
& \leq  \E\left[\mathbf{1}\left[\kappa^* = K\right] \cdot   \left| (x^*)^\top (\hat{\theta}_{\zeta, t} - \theta)\right| \right] \nonumber\\
& \;\;\;\;+ \E\left[\mathbf{1}\left[\kappa^* < K\right] 
\times\mathbf{1}\left[ \left| (x^*)^\top (\hat{\theta}_{\zeta, t} - \theta)\right| 
\geq  C \sqrt{d} \cdot \alpha_{\zeta, t}^{x^*} \omega_{\zeta, t}^{x^*}  \right] 
\right.\nonumber\\
&\left.\qquad\quad\times \left| (x^*)^\top (\hat{\theta}_{\zeta, t} - \theta)\right|\right] . \label{eq:lem-reward-error-bound-3}
\end{align}

We first focus on the first term in the Right-Hand Side of Eq.~\eqref{eq:lem-reward-error-bound-3}.  When $\kappa^* = K$, we have $\omega_{\zeta, t}^{x^*} \leq 2/(T/\delta)^2$. Therefore, 
\begin{align*}
&\Pr\left[\kappa^* = K ~\text{and}~  \left| (x^*)^\top (\hat{\theta}_{\zeta, t} - \theta)\right| 
\geq  \delta/\sqrt{T  \ln^4 T \ln^2(1/\delta)} \right] \\
& \leq \Pr\left[   (\omega_{\zeta, t}^{x^*})^{-1} \left| (x^*)^\top (\hat{\theta}_{\zeta, t} - \theta)\right| >T^{1.5}/(2 \delta \ln^2 T \ln (1/\delta))   \right] \\
&\lesssim \delta^3 \cdot \exp(-T),
\end{align*}
where the last inequality is due to Corollary~\ref{cor:uniform-tail-bound} and for $T \gtrsim \sqrt{d}$. Therefore, we have
\begin{align}
 &\E\left[\mathbf{1}\left[\kappa^* = K\right] \left| (x^*)^\top (\hat{\theta}_{\zeta, t} - \theta)\right| \right]   \lesssim  \delta /\sqrt{T \ln^4 T \ln^2(1/\delta)} 
 \nonumber \\ &
 \;\;\;\; + \E\left[\mathbf{1}\left[\kappa^* = K ~\text{and}~ \left| (x^*)^\top (\hat{\theta}_{\zeta, t} - \theta)\right| 
\right.\right.\nonumber\\&\left.\left.\;\;\;\;
> \delta/\sqrt{T \ln^4 T \ln^2(1/\delta)}\right]  \cdot  \left| (x^*)^\top (\hat{\theta}_{\zeta, t} - \theta)\right|  \right]  \nonumber\\
 &\leq  \delta/ \sqrt{T \ln^4 T \ln^2(1/\delta)} \nonumber \\
 &  \;\;\;\;+  \left\{\Pr\left[\kappa^* = K ~\text{and}~ \left| (x^*)^\top (\hat{\theta}_{\zeta, t} - \theta)\right|
 \right.\right.\nonumber\\&\left.\left.
  > \delta/\sqrt{T \ln^4 T \ln^2(1/\delta)}\right]  \cdot  \E\left[ \left| (x^*)^\top (\hat{\theta}_{\zeta, t} - \theta)\right|^2  \right] \right\}^{1/2}\nonumber\\
& \lesssim  \delta / \sqrt{T \ln^4 T \ln^2(1/\delta)},  \label{eq:lem-reward-error-bound-6}
\end{align}
where the second inequality due to Cauchy-Schwartz, and the last inequality is because of 
\begin{align*}
\E\left[ \left| (x^*)^\top (\hat{\theta}_{\zeta, t} - \theta)\right|^2\right]& \leq \E\left[ \left\| \hat{\theta}_{\zeta, t} - \theta\right\|_2^2\right]
\\&
 \leq \E\left[\left( \left\| \hat{\theta}_{\zeta, t}\right\|_2 + 1\right)^2\right] \lesssim T^2 .
\end{align*}

Now we work with the second term in the Right-Hand Side of Eq.~\eqref{eq:lem-reward-error-bound-3}.  When $\kappa^* < K$, we have  $2^{-\kappa^*} < \omega_{\zeta, t}^{x^*}$ and $\alpha_t^{\kappa^*} \leq \alpha_{\zeta, t}^{x^*}$, and therefore 
\begin{align}
&\mathbf{1}\left[\kappa^* < K, \left| (x^*)^\top (\hat{\theta}_{\zeta, t} - \theta)\right| 
\geq  C \sqrt{d} \cdot \alpha_{\zeta, t}^{x^*} \cdot \omega_{\zeta, t}^{x^*}  \right] 
\nonumber\\&\qquad
\times  \left| (x^*)^\top (\hat{\theta}_{\zeta, t} - \theta)\right| \nonumber \\
&\leq\  \mathbf{1}\left[ \left| (x^*)^\top (\hat{\theta}_{\zeta, t} - \theta)\right| 
\geq C \sqrt{d} \cdot \alpha_{\zeta, t}^{x^*}\cdot 2^{-\kappa^*}  \right] 
\nonumber\\&\qquad
\cdot \left| (x^*)^\top (\hat{\theta}_{\zeta, t} - \theta)\right| \nonumber \\
&\leq \mathbf{1}\left[ m_{t}^{\kappa^*}
\geq C \sqrt{d} \cdot \alpha_{\zeta, t}^{x^*}
\cdot 2^{-\kappa^*}  \right] \cdot m_t^{\kappa^*} 
\nonumber\\&
\leq \sum_{\kappa = 1}^{K-1} \mathbf{1}\left[ m_{t}^{\kappa}
\geq C \sqrt{d} \cdot \alpha_{\zeta, t}^{x^*} \cdot 2^{-\kappa^*}  \right] \cdot m_{\zeta, t}^\kappa . \label{eq:lem-reward-error-bound-4}
\end{align}
Taking expectation and invoking Eq.~\eqref{eq:lem-reward-error-bound-2}, we have 
\begin{align}
&\E\left[\mathbf{1}\left[\kappa^* < K, 
 \left| (x^*)^\top (\hat{\theta}_{\zeta, t} - \theta)\right| 
\geq  C \sqrt{d} \cdot \alpha_{\zeta, t}^{x^*} \cdot \omega_{\zeta, t}^{x^*}  \right] 
\right.\nonumber\\&\qquad\left.
\cdot \left| (x^*)^\top (\hat{\theta}_{\zeta, t} - \theta)\right|\right] \nonumber\\
&\leq \sum_{\kappa = 1}^{K - 1} 
\E\left[ \mathbf{1}\left[ m_{t}^{\kappa}
\geq C \sqrt{d} \cdot \alpha_{\zeta, t}^{x^*} \cdot 2^{-\kappa^*}  \right] \cdot m_{\zeta, t}^\kappa\right] 
\nonumber\\&
\lesssim \delta d / \sqrt{T \ln^2 T} .
 \label{eq:lem-reward-error-bound-5}
\end{align}
Combining Eq.~\eqref{eq:lem-reward-error-bound-3}, Eq.~\eqref{eq:lem-reward-error-bound-6}, and Eq.~\eqref{eq:lem-reward-error-bound-5}, we have
\begin{align*}
&\E\left[\mathbf{1}\left[ \left| (x^*)^\top (\hat{\theta}_{\zeta, t} - \theta)\right| 
\geq C \sqrt{d} \cdot \alpha_{\zeta, t}^{x^*} \cdot \omega_{\zeta, t}^{x^*} \right]
\right.\nonumber\\&\qquad\left.
\times \left| (x^*)^\top (\hat{\theta}_{\zeta, t} - \theta)\right|\right] 
\lesssim \delta d / \sqrt{T \ln^2 T},
\end{align*}
which is to be demonstrated.
\end{proof}

For any layer $\zeta \in \{0, 1, 2, \dots, \zeta_0\}$ and any time period $t \in [T]$, we define 
\[
\overline m_{\zeta,t} := \max_{x\in\mathcal N_{\zeta,t}} \{ x^\top\theta \}, \;\;\; \text{and}\;\;\;
\underline m_{\zeta,t} := \min_{x\in\mathcal N_{\zeta,t}} \{ x^\top\theta \}
\]
as the largest and smallest mean reward for actions in the action subset $\mathcal N_{\zeta,t}$.
For convenience, we also define 
\[
\overline m_{\zeta_0+1,t} := \overline m_{\zeta_0,t},  \;\;\; \text{and}\;\;\; \underline m_{\zeta_0+1,t}:=\underline m_{\zeta_0,t}.
\]

Note that $\max_{x\in D_t} \{x^\top\theta \} = \overline m_{0,t}$ and $x_{t}^\top\theta \geq \underline m_{\zeta_t,t}$ (due to $x_t\in\mathcal N_{\zeta_t,t}$), we have that the regret incurred at time $t$ is
\begin{align}\label{eq:regret-time-t-mt}
\max_{x\in D_t}&\{x^\top\theta\} - x_t^\top \theta \leq  \left((\overline m_{0,t}-\overline m_{\zeta_t,t}) + (\overline m_{\zeta_t,t}-\underline m_{\zeta_t,t})\right) 
\nonumber\\&
\leq (\overline m_{\zeta_t,t}-\underline m_{\zeta_t,t}) + \sum_{\zeta = 1}^{\zeta_t} (\overline m_{\zeta-1,t}-\overline m_{\zeta,t}) .
\end{align}

In the following lemma, we provide upper bounds for the expressions in the Right-Hand Side of Eq.~(\ref{eq:regret-time-t-mt}). 
Intuitively, the following lemma states that the expected maximum rewards between adjacent layers are close to each other, 
and the expected differences between any pair of actions inside any layer are small and exponentially decreasing as the layer increases.  

\begin{lemma}\label{lem:mt}
For all $t$ and $\zeta=0,1,\cdots,\zeta_0$, it holds that
\begin{align}
& \mathbb E\left[\max\{\overline m_{\zeta,t}-\overline m_{\zeta+1,t} ,0\}\right] 
\lesssim \delta d / \sqrt{T \ln^2 T};\label{eq:mt-1}\\
& \mathbb E\left[ \max\{\overline m_{\zeta,t}-\underline m_{\zeta,t} - 2^{3-\zeta}, 0\}\cdot\vct 1\{\zeta \leq \zeta_t\}\right] \nonumber\\
&\qquad\qquad\qquad\qquad\qquad\qquad
\lesssim \delta  d / \sqrt{T \ln^2 T}.\label{eq:mt-2}
\end{align}
\end{lemma}

\begin{proof}
We first prove Eq.~(\ref{eq:mt-1}). Let 
\[
y^*_t := \arg\max_{y\in\mathcal N_{\zeta,t}}\{y^\top\theta\}
\]
and
\[
z^*_t := \arg\max_{z\in\mathcal N_{\zeta,t}}\{z^\top\hat\theta_{\zeta,t} \}.
\]
If $y^*_t\in\mathcal N_{\zeta+1,t}$, then $\overline m_{\zeta,t}=\overline m_{\zeta+1,t}$ because $\mathcal N_{\zeta+1,t}\subseteq\mathcal N_{\zeta,t}$.
On the other hand, if $y^*_t\notin\mathcal N_{\zeta+1,t}$,
note that $z^*_t\in\mathcal N_{\zeta+1,t}$ because $z^*_t$ maximizes $z^\top\hat\theta_{\zeta,t}$
in $\mathcal N_{\zeta,t}$.
Summarizing both cases of $y^*_t\in\mathcal N_{\zeta+1,t}$ (in which $\overline m_{\zeta+1,t}=\overline m_{\zeta,t}$)
and $y^*_t\notin\mathcal N_{\zeta+1,t}$ (in which $\overline m_{\zeta+1,t}\geq (z^*_t)^\top\theta$ as $z^*_t\in\mathcal N_{\zeta+1,t}$), we have
\begin{equation}
\overline m_{\zeta,t} -\overline m_{\zeta+1, t} \leq \vct 1\{y^*_t\notin\mathcal N_{\zeta+1,t}\}\cdot (y^*_t-z^*_t)^\top\theta.
\label{eq:mm1-eq1}
\end{equation}

For any $\zeta,t$ and $y\in\mathcal N_{\zeta,t}$, define
\[
\mathcal E_{\zeta,t}^y := \{|y^\top(\hat\theta_{\zeta,t}-\theta)| \leq \varpi_{\zeta,t}^y\}
\]
as the success event in which the estimation error of $y^\top\hat\theta_{\zeta,t}$ for $y^\top\theta$ is within the confidence interval $\varpi_{\zeta,t}^y$.
By definition, 
\begin{align}
(y_t^*)^\top\theta &\leq (y_t^*)^\top\hat\theta_{\zeta,t} + \varpi_{\zeta,t}^{y^*_t} + \vct 1\{\neg\mathcal E_{\zeta,t}^{y^*_t}\}\cdot \big|(y_t^*)^\top(\hat\theta_{\zeta,t}-\theta)\big|;\label{eq:xistar-expansion}\\
(z_t^*)^\top\theta &\geq (z_t^*)^\top\hat\theta_{\zeta,t} - \varpi_{\zeta,t}^{z^*_t} - \vct 1\{\neg\mathcal E_{\zeta,t}^{z^*_t}\}\cdot \big|(z_t^*)^\top(\hat\theta_{\zeta,t}-\theta)\big|.
\label{eq:xjstar-expansion}
\end{align}
Also, conditioned on the event $y^*_t\notin\mathcal N_{\zeta+1,t}$, the procedure of Algorithm \ref{alg:suplinucb} implies 
\begin{equation}
(y_t^*)^\top\hat\theta_{\zeta,t} < (z_t^*)^\top\hat\theta_{\zeta,t} - 2^{1-\zeta}.
\label{eq:xijstar-side}
\end{equation}
Subtracting Eq.~(\ref{eq:xjstar-expansion}) from Eq.~(\ref{eq:xistar-expansion}) and considering Eq.~(\ref{eq:xijstar-side}), we have
\begin{align}
&(y_t^*-z_t^*)^\top\theta
\nonumber\\
\leq& \varpi_{\zeta,t}^{y^*_t} + \varpi_{\zeta,t}^{z^*_t} - 2^{1-\zeta}
+ \sum_{x\in\{y^*_t,z^*_t\}}\vct 1\{\neg\mathcal E_{\zeta,t}^{x}\}\cdot \big|x^\top(\hat\theta_{\zeta,t}-\theta)\big|\nonumber\\
\leq &\sum_{x\in\{y^*_t,z^*_t\}}\vct 1\{\neg\mathcal E_{\zeta,t}^{x}\}\cdot \big|x^\top(\hat\theta_{\zeta,t}-\theta)\big|,
\label{eq:xijstar-main}
\end{align}
where the last inequality holds because $\varpi_{\zeta,t}^x\leq 2^{-\zeta}$ for all $x\in\mathcal N_{\zeta,t}$, if the algorithm is executed to level $\zeta+1$.
Combining Eqs.~(\ref{eq:mm1-eq1},\ref{eq:xijstar-main}) and Lemma \ref{lem:reward-error-bound-in-each-layer}, 
taking expectations, we obtain
\begin{align*}
&\mathbb E\left[\max\{\overline m_{\zeta,t}-\overline m_{\zeta+1,t}, 0\}\right] \\
\leq\ & \mathbb E\left[\vct 1\{y^*_t\notin\mathcal N_{\zeta+1,t}\}\cdot \left(\vct 1\{\neg\mathcal E_{\zeta,t}^{y^*_t}\} \big|(y_t^*)^\top(\hat\theta_{\zeta,t}-\theta)\big|
\right.\right.\nonumber\\&\qquad\left.\left.
+\vct 1\{\neg\mathcal E_{\zeta,t}^{z^*_t}\} \big|(z_t^*)^\top(\hat\theta_{\zeta,t}-\theta)\big|\right)\right]\nonumber\\
\leq\ & \mathbb E\left[\vct 1\{\neg\mathcal E_{\zeta,t}^{y^*_t}\} \big|(y_t^*)^\top(\hat\theta_{\zeta,t}-\theta)\big|\right]
\nonumber\\&\qquad
+\mathbb E\left[\vct 1\{\neg\mathcal E_{\zeta,t}^{z^*_t}\} \big|(z_t^*)^\top(\hat\theta_{\zeta,t}-\theta)\big|\right] \nonumber
\nonumber\\
\lesssim &\,\delta d / \sqrt{T \ln^2 T}.
\end{align*}

Now we focus on Eq.~(\ref{eq:mt-2}). We only need to prove the equation for $\zeta > 0$ since it is trivially true for $\zeta = 0$. Let 
\[w^*_t := \arg\min_{w\in\mathcal N_{\zeta,t}}\{w^\top\theta\}.
\]
Clearly, we have that
\[\overline m_{\zeta,t}-\underline m_{\zeta,t}=(y_t^*-w_t^*)^\top\theta.
\]
Similar to Eqs.~(\ref{eq:xistar-expansion},\ref{eq:xjstar-expansion}), we can establish that
\begin{align}
(y_t^*)^\top\theta &\leq (y_t^*)^\top\hat\theta_{\zeta-1,t} + \varpi_{\zeta-1,t}^{y^*_t}
\nonumber\\&\qquad
+ \vct 1\{\neg\mathcal E_{\zeta-1,t}^{y^*_t}\}\cdot \big|(y_t^*)^\top(\hat\theta_{\zeta-1,t}-\theta)\big|;\label{eq:xistar-expansion-2}\\
(w_t^*)^\top\theta &\geq (w_t^*)^\top\hat\theta_{\zeta-1,t} - \varpi_{\zeta-1,t}^{w^*_t}
\nonumber\\&\qquad
- \vct 1\{\neg\mathcal E_{\zeta-1,t}^{w^*_t}\}\cdot \big|(w_t^*)^\top(\hat\theta_{\zeta-1,t}-\theta)\big|.
\label{eq:xjstar-expansion-2}
\end{align}
In addition, because both $y^*_t$ and $w^*_t$ belong to $\mathcal N_{\zeta,t} \subseteq \mathcal N_{\zeta - 1,t}$, the second step of Algorithm \ref{alg:suplinucb} implies that conditional on $\zeta \leq \zeta_t$, 
\begin{align}
(y_t^*)^\top\hat\theta_{\zeta-1,t} &\leq (w_t^*)^\top\hat\theta_{\zeta-1,t} - 2^{1-(\zeta-1)} 
\nonumber\\&
\leq (w_t^*)^\top\hat\theta_{\zeta-1,t}-2^{2-\zeta},
\label{eq:xijstar-side-2}
\end{align}
and
\begin{equation}
 \varpi_{\zeta-1,t}^{y^*_t} \leq 2^{-(\zeta -1)}, \qquad  \varpi_{\zeta-1,t}^{w^*_t} \leq 2^{-(\zeta -1)} .
\label{eq:xijstar-side-3}
\end{equation}
Subtracting Eq.~(\ref{eq:xistar-expansion-2}) from Eq.~(\ref{eq:xjstar-expansion-2}) and applying Eqs.~(\ref{eq:xijstar-side-2},\ref{eq:xijstar-side-3}), we get for any $\zeta \leq \zeta_t$,
\begin{align*}
& \overline m_{\zeta,t} -\underline m_{\zeta,t}  = (y_t^*-w_t^*) \\
& \leq 2^{2-\zeta} +  2 \cdot 2^{-(\zeta -1)}  +  \vct 1\{\neg\mathcal E_{\zeta-1,t}^{y^*_t}\}\cdot \big|(y_t^*)^\top(\hat\theta_{\zeta-1,t}-\theta)\big| \\
&\qquad\qquad +  1\{\neg\mathcal E_{\zeta-1,t}^{w^*_t}\}\cdot \big|(w_t^*)^\top(\hat\theta_{\zeta-1,t}-\theta)\big| \\
& = 2^{3-\zeta}  +  \vct 1\{\neg\mathcal E_{\zeta-1,t}^{y^*_t}\}\cdot \big|(y_t^*)^\top(\hat\theta_{\zeta-1,t}-\theta)\big| \\
&\qquad\qquad + 1\{\neg\mathcal E_{\zeta-1,t}^{w^*_t}\}\cdot \big|(w_t^*)^\top(\hat\theta_{\zeta-1,t}-\theta)\big| .
\end{align*}
Therefore, since the right hand-side of the above inequality is non-negative, we have
\begin{align*}
&\E\left[\max\{ \overline m_{\zeta,t} -\underline m_{\zeta,t}  - 2^{3-\zeta}, 0\}\cdot\vct 1\{\zeta \leq \zeta_t\}\right] \\
&\leq \E\left[ \vct 1\{\neg\mathcal E_{\zeta-1,t}^{y^*_t}\}
\times  \big|(y_t^*)^\top(\hat\theta_{\zeta-1,t}-\theta)\big| 
\right.\\
&\qquad\qquad\left. 
+  \vct 1\{\neg\mathcal E_{\zeta-1,t}^{w^*_t}\}\cdot \big|(w_t^*)^\top(\hat\theta_{\zeta-1,t}-\theta)\big| \right] .
\end{align*}
We finally apply Lemma \ref{lem:reward-error-bound-in-each-layer} and prove Eq.~(\ref{eq:mt-2}).
\end{proof}

\subsection{The elliptical potential lemma, and putting everything together}

\begin{lemma}\label{lem:single-step-regret}
If the parameter $C$ in Algorithm~\ref{alg:suplinucb} is a large enough universal constant, then we have
\begin{align}
\E\left[\max\left\{R_T - 8 \cdot \sum_{t=1}^{T} \varpi_{\zeta_t, t}^{x_t}, 0\right\} \right] \lesssim \delta d\sqrt{T} .
\end{align}
\end{lemma}
Note that instead of a high probability bound, which is usual in the previous analysis (e.g., \cite{dani2008stochastic,abbasi2011improved}), our upper bound is in an expectation form. This crucially helps us to avoid the extra $\log T$ factor due to the union bound argument.

\begin{proof}[Proof of Lemma \ref{lem:single-step-regret}]
Since 
\[
R_T = \sum_{t=1}^T (\max_{x \in D_t} \{x^\top\theta\} - x_t ^\top\theta),
\]
by Eq.~(\ref{eq:regret-time-t-mt}) we have
\begin{align}
R_T \leq \sum_{t=1}^T \left( (\overline m_{0,t}-\overline m_{\zeta_t,t}) + (\overline m_{\zeta_t,t}-\underline m_{\zeta_t,t})\right).  \label{eq:rt-0}
\end{align}
By Eq. (\ref{eq:mt-2}) in Lemma~\ref{lem:mt}, we have that for any time $t$, 
\begin{align*}
&\mathbb E\left[\max\{\overline m_{\zeta_t,t}-\underline m_{\zeta_t,t}, - 2^{3-\zeta_t}, 0\}\right]\\
& \leq \sum_{\zeta = 0}^{\zeta_0}
\mathbb E\left[\max\{\overline m_{\zeta,t}-\underline m_{\zeta,t} - 2^{3-\zeta}, 0\}\cdot\vct 1\{\zeta \leq \zeta_t\}\right]\\
&\lesssim \delta d  / \sqrt{T},
\end{align*}


Together with Eq.~(\ref{eq:mt-1}) in Lemma~\ref{lem:mt}, we have that 
\begin{align}
&\mathbb E\left[\max\{(\overline m_{0,t}-\overline m_{\zeta_t,t}) + (\overline m_{\zeta_t,t}-\underline m_{\zeta_t,t}) - 2^{3 - \zeta_t}, 0\}\right] \nonumber \\
& \leq  \sum_{\zeta=0}^{\zeta_0}E\left[\max\{\overline m_{\zeta,t}-\overline m_{\zeta+1,t}, 0\}\right] \nonumber\\
&\qquad+ \mathbb E\left[\max\{\overline m_{\zeta_t,t}-\underline m_{\zeta_t,t}, - 2^{3-\zeta_t}, 0\}\right] \lesssim \delta d  / \sqrt{T} . \label{eq:rt-0a}
\end{align}
Summing up \eqref{eq:rt-0a} for all $t \in [T]$ and together with \eqref{eq:rt-0}, we have that 
\begin{align}
\mathbb E\left[\max\left\{R_T - \sum_{t=1}^{T} 2^{3-\zeta_t} , 0\right\}\right] \lesssim \delta d \sqrt{T} . \label{eq:rt-1}
\end{align}
Note that $\varpi_{\zeta,t}^{x_t}\geq 2^{-\zeta_t}- \delta \sqrt{ d/T}$ by the first and the third cases of the if-elseif-else loop of Algorithm \ref{alg:suplinucb}. Therefore, Eq.~(\ref{eq:rt-1}) implies the lemma statement.
\end{proof}

Below we state a version of the celebrated \emph{elliptical potential lemma}, key to many existing analysis of linearly parameterized bandit problems \citep{auer2002using,filippi2010parametric,abbasi2011improved,chu2011contextual,li2017provable}.
\begin{lemma}[\citet{abbasi2011improved}]
\label{lem:elliptical}
 Let $U_0 = I$ and $U_t = U_{t-1} + y_t y_t^\top$ for $t \geq 1$. For any vectors $y_1, y_2, \dots, y_T$, it holds that 
\[
\sum_{t = 1}^{T} y_t^\top U_{t-1}^{-1}y_t 
\leq 2 \ln (\mathrm{det}(U_T)) .
\]
\end{lemma} 

Using Lemma~\ref{lem:elliptical}, we prove the following Lemma~\ref{lem:elliptical-jensen}. The proof Lemma~\ref{lem:elliptical-jensen} follows the similar lines of Lemma 6 in \citep{li2019near} and we defer it to Appendix~\ref{app:proofs}. At a high level, the proof exploits the power of variated confidence levels (i.e., the specially designed $\alpha^i_{\zeta, t}$ quantity in Algorithm~\ref{alg:suplinucb}) and relies on an application of Jensen's inequality to the concave function $f(\tau) = \sqrt{\tau \ln ((T \ln^4 T \ln^2 (1/\delta)) \tau/(d\delta^2))}$, as well as the commonly used $f(\tau) = \sqrt{\tau}$.
\begin{lemma} \label{lem:elliptical-jensen}
It holds that
\[
\displaystyle{\sum_t \varpi_{\zeta_t, t}^{x_t} \lesssim d \sqrt{ T \log T \log (1/\delta)} \cdot \log \log (T/\delta)}.
\]
\end{lemma}

Combining Lemma~\ref{lem:single-step-regret} and Lemma~\ref{lem:elliptical-jensen}, we prove Theorem~\ref{thm:main}.

\section{Conclusions}

In this paper we study the linearly parameterized contextual bandit problem and develop algorithms that achieve minimax-optimal regret up to iterated logarithmic terms.
Future directions include generalizing the proposed approach to contextual bandits with generalized linear models, as well as other variants of contextual bandit problems.

\section*{Acknowledgment}

Xi Chen would like to thank the support from NSF IIS-1845444. Yuan Zhou would like to thank the support from NSF CCF-2006526.

\bibliographystyle{apalike}
\bibliography{refs}

\appendix

\section{Useful probability tools}

A separable process \footnote{See Definition 5.22 in \citet{handel2014probability} for a technical definition of separable stochastic processes.} $\{G_\phi\}_{\phi\in\Theta}$
with respect to a metric space $(\Theta,d)$
is \emph{sub-Gaussian} if for any $\lambda\in\mathbb R$ and $\phi,\phi'\in\Theta$, $\mathbb E[e^{\lambda(X_\phi-X_{\phi'})}] \leq e^{\lambda^2d^2(\phi,\phi')/2}$.
Let also $\diam(\Theta) = \sup_{\phi,\phi'\in\Theta}d(\phi,\phi')$ be the \emph{diameter} of the metric space $(\Theta,d)$.
The following result is cited from \citep[Theorem 5.29]{handel2014probability}.
\begin{lemma}
There exists a universal constant $C_0<\infty$ such that
for all $z>0$ and $\phi_0\in\Theta$, 
\begin{align*}
&\Pr\left[\sup_{\phi\in\Theta} G_\phi-G_{\phi_0} \geq C_0\int_0^\infty\sqrt{\ln N(\Theta;d,\epsilon)}\ud\epsilon + z\right]\\
& \leq C_0e^{-z^2/(C_0\cdot \diam(\Theta))},
\end{align*}
where $N(\Theta;d,\epsilon)$ is the covering number of the metric space $(\Theta,d)$ up to precision $\epsilon$.
\label{lem:gp-tail}
\end{lemma}

\section{Omitted proofs in Section~\ref{lem:single-step-regret}} \label{app:proofs}

\begin{proof}[Proof of Lemma~\ref{lem:elliptical-jensen}]
Let $\calT_{\zeta}$ be all time periods $t$ such that $\zeta_t=\zeta$, and define $T_\zeta = |\calT_{\zeta}|$.
We have
\begin{align}
\sum_t \varpi_{\zeta_t, t}^{x_t} 
  \lesssim \sqrt{d} \cdot 
\sum_{\zeta} \sum_{t \in \calT_{\zeta}} 
\alpha_{\zeta, t}^{x_t} \omega_{\zeta, t}^{x_t} . \label{eq:thm-main-1}
\end{align}

First by Lemma~\ref{lem:elliptical}, we have
\begin{align}
 \sum_{t \in \calT_{\zeta}} (\omega_{\zeta, t}^{x_t})^2 
\leq \ln (\det(\Lambda_{T_\zeta}))
\lesssim d \ln (T_\zeta/d),  \label{eq:thm-main-2}
\end{align}
where the last inequality is due to 
\begin{align}
\det(\Lambda_{T_\zeta}) \leq \mathrm{tr}(\Lambda_{T_\zeta} /d)^d \leq ((T_\zeta+1)/d)^d . \label{eq:thm-main-2a}
\end{align}

Let us now focus on the Right-Hand Side of Eq.~\eqref{eq:thm-main-1}, let 
\[
\calT_{\zeta}^+ := 
\left\{t \in \calT_{\zeta} : \omega_{\zeta, t}^{x_t} \geq \sqrt{d\delta^2 /(T\ln^4 T \ln^2 (1/\delta))}\right\}
\]
and let 
\begin{align*}
\calT_{\zeta}^- &:= 
\left\{t \in \calT_{\zeta} : \omega_{\zeta, t}^{x_t} < \sqrt{d \delta^2/(T\ln^4 T \ln^2 (1/\delta))}\right\}\\
&= \calT_{\zeta} \setminus \calT_{\zeta}^+ .
\end{align*}
We have that
\begin{align}
&\sum_{t \in \calT_{\zeta}} \alpha_{\zeta, t}^{x_t} \omega_{\zeta, t}^{x_t} = \sum_{t \in \calT_{\zeta}^+} \alpha_{\zeta, t}^{x_t} \omega_{\zeta, t}^{x_t} 
+ \sum_{t \in \calT_{\zeta}^-} \alpha_{\zeta, t}^{x_t} \omega_{\zeta, t}^{x_t} \nonumber \\
&  = \sum_{t \in \calT_{\zeta}^+} \sqrt{\ln ((T \ln^4 T \ln^2 (1/\delta)) (\omega_{\zeta, t}^{x_t})^2 / (d\delta^2))}\omega_{\zeta, t}^{x_t}\nonumber \\
& \qquad
+ \sum_{t \in \calT_{\zeta}^-}  \omega_{\zeta, t}^{x_t} \nonumber \\
& \leq  \sum_{t \in \calT_{\zeta}^+} \sqrt{\ln ((T \ln^4 T \ln^2 (1/\delta)) (\omega_{\zeta, t}^{x_t})^2 / (d\delta^2))}\omega_{\zeta, t}^{x_t} ]\nonumber \\
& \qquad
+ T_{\zeta} \sqrt{d \delta^2 / (T\ln^4 T \ln^2 (1/\delta))} . \label{eq:thm-main-3}
\end{align}
Note that the univariate function $f(\tau) = \sqrt{\tau \ln ((T \ln^4 T \ln^2 (1/\delta)) \tau / (d \delta^2)} $ is concave for $\tau \geq d \delta^2/(T \ln^4 T \ln^2 (1/\delta))$. Applying Jensen's inequality to $f(\tau)$ with $\tau = (\omega_{\zeta, t}^{x_t})^2$ ($t \in \calT_{\zeta}^+$), we have
\begin{align}
&\sum_{t \in \calT_{\zeta}^+} \sqrt{\ln ((T \ln^4 T \ln^2 (1/\delta)) (\omega_{\zeta, t}^{x_t})^2 / (d\delta^2))}\omega_{\zeta, t}^{x_t} \nonumber \\
 \leq\ & |\calT_{\zeta}^+| \cdot \sqrt{\frac{\sum_{t \in \calT_{\zeta}^+} (\omega_{\zeta, t}^{x_t})^2}{|\calT_{\zeta}^+|}}\nonumber \\
& \qquad
 \times \sqrt{\ln \left(\frac{T \ln^4 T \ln^2 (1/\delta)}{d \delta^2} \cdot \frac{\sum_{t \in \calT_{\zeta}^+} (\omega_{\zeta, t}^{x_t})^2}{|\calT_{\zeta}^+|}\right)} \nonumber \\
 \lesssim \ & \sqrt{ |\calT_{\zeta}^+| d \ln (|\calT_{\zeta}|/d) \ln  \left(\frac{T \ln^4 T \ln^2 (1/\delta)}{d \delta^2} \cdot \frac{d \ln (|\calT_{\zeta}| / d)}{|\calT_{\zeta}^+|}\right)} \nonumber \\
 \lesssim \ & \sqrt{d T_{\zeta}  \ln (T_{\zeta}/d) \ln  \left(\frac{T \ln^4 T \ln^2 (1/\delta)}{d \delta^2} \cdot \frac{d \ln (T_{\zeta} / d)}{T_{\zeta}}\right)} \nonumber \\
 \lesssim \ &  \sqrt{d T_{\zeta} \ln (T_{\zeta}/d) \ln (T \ln^5 T / (T_{\zeta} \delta^3))} , \label{eq:thm-main-4}
\end{align}
where the second inequality is due to Lemma~\ref{lem:elliptical} and Eq.~\eqref{eq:thm-main-2}, and the third inequality is due to the monotonicity of the function $g(x) = \sqrt{x d \ln (T_{\zeta}/d) \ln ((T \ln^4 T \ln^2 (1/\delta))/(d\delta^2) \cdot (d \ln (T_{\zeta}/d) / x))}$ for large enough $x$. Combining Eq.~\eqref{eq:thm-main-3}, and Eq.~\eqref{eq:thm-main-4}, we have
\begin{align}
\sum_{t \in \calT_{\zeta}} \alpha_{\zeta, t}^{x_t} \omega_{\zeta, t}^{x_t} 
&\lesssim \sqrt{d T_{\zeta} \ln (T_{\zeta}/d) \ln (T \ln^5 T / (T_{\zeta} \delta^3))}\nonumber \\
& \qquad
+ T_{\zeta} \delta \sqrt{d / (T\ln^4 T \ln^2 (1/\delta))} . \label{eq:thm-main-5}
\end{align}

By Algorithm \ref{alg:suplinucb}, we know that 
$\varpi_{\zeta,t}^{x_t}
=\sqrt{d} \cdot \alpha_{\zeta,t}^{x_t}\omega_{\zeta,t}^{x_t} 
\geq 2^{1-\zeta}$ 
for all $t\in\mathcal T_{\zeta}$.
Subsequently, 
\begin{align*}
(2^{-\zeta-1})^2\cdot T_\zeta 
&\leq \sum_{t\in\mathcal T_{\zeta}}(\varpi_{\zeta,t}^{x_t})^2 
\leq \sqrt{d} \cdot \max_{t\in\mathcal T_{\zeta}}(\alpha_{\zeta,t}^{x_t})^2\cdot 
\sum_{t\in\mathcal T_{\zeta}}(\omega_{\zeta,t}^{x_t})^2\\
& \lesssim \sqrt{d} \cdot \log(T \ln ^4 T  \ln^2 (1/\delta) /(d\delta^2)) \cdot d\log T,
\end{align*}
where the last inequality holds by applying Lemma \ref{lem:elliptical}.
Therefore, 
\begin{align}\label{eq:bound-on-size}
T_\zeta \lesssim 
4^\zeta \cdot d^{3/2} \log T \log (T/\delta). 
\end{align}

We first divide the resolution levels $\zeta\in\{0,1,\cdots,\zeta_0\}$ into two different sets:
$\mathcal Z_1 := \{0,1,\cdots,\zeta^*\}$ and  $\mathcal Z_2 := \{\zeta^*<\zeta\leq\zeta_0\}$, where $\zeta^*$ is an integer to be defined later.
Clearly $\mathcal Z_1$ and $\mathcal Z_2$ partition $\{0,\cdots,\zeta_0\}$.
Note that $\sqrt{d} \cdot \sum_{t\in\mathcal T_{\zeta}} \alpha_{\zeta,t}^{x_t}\omega_{\zeta,t}^{x_t}\lesssim 2^{-\zeta}T_\zeta$ 
because $\varpi_{\zeta,t}^{x_t}\leq 2^{1-\zeta}$ for all $t\in\mathcal T_{\zeta}$.
\begin{align}
\sqrt{d}  \sum_{\zeta\in\mathcal Z_1} \sum_{t\in\mathcal T_\zeta}
&\alpha_{\zeta,t}^{x_t}\omega_{\zeta,t}^{x_t} 
\lesssim \sum_{\zeta=0}^{\zeta^*}2^{-\zeta}
\cdot 4^{\zeta} \cdot d^{3/2} \log T  \log (T/\delta)\nonumber \\
& \leq 2^{\zeta^*+1}\cdot d^{3/2} \log T \log (T/\delta);\label{eq:zeta-case1}
\end{align}
\begin{align}
&\sqrt{d}  \sum_{\zeta\in\mathcal Z_2}\sum_{t\in\mathcal T_{\zeta}}\alpha_{\zeta,t}^{x_t}\omega_{\zeta,t}^{x_t} \nonumber \\
& \lesssim d\sum_{\zeta \in \mathcal Z_2 }  \sqrt{T_\zeta \log(T)\log(T \log^5 T/(T_\zeta \delta^3)) } 
+ \delta d\sqrt{T} / \log^2 T\ \nonumber\\
&\leq d\sqrt{\left|\mathcal Z_2\right| \left(\sum_{\zeta \in \mathcal Z_2} T_\zeta\right) \log  (T) \log \left(T \log^5 T \cdot \frac{\left|\mathcal Z_2\right|}{ \delta^3 \sum_{\zeta \in \mathcal Z_2} T_\zeta}\right)} \nonumber \\
& \qquad
+ \delta d\sqrt{T} / \log^2 T\nonumber \\
&\lesssim  d\sqrt{\left|\mathcal Z_2\right| T \log  (T) \log \left(\log^5 T \left|\mathcal Z_2\right| /\delta^3 \right)} 
+ \delta d\sqrt{T} / \log^2 T,
\label{eq:zeta-case2}
\end{align}
where the inequality above Eq.~\eqref{eq:zeta-case2} is because of the concavity of the function $\sqrt{x \ln (T \log^5 T|\mathcal Z_2| / (x \delta^3))}$ and Jensen's inequality, and Eq.~\eqref{eq:zeta-case2}  is due to $\sum_{\zeta \in \mathcal Z_2} T_\zeta \leq T$ and the monotonicity of the function $\sqrt{x \ln (T \log^5 T |\mathcal Z_2|/(x\delta^3))}$. 

Recall that $ \sqrt{T/d} / \delta \leq 2^{\zeta_0}\leq 2  \sqrt{T/d} / \delta$. 
Select $\zeta^* = \zeta_0 - \lfloor \log_2 (\ln(T) \ln(T/\delta) /\delta)\rfloor$; 
we have that $|\mathcal Z_2| = O(\log \log (T/\delta) + \log (1/\delta))$ and 
$2^{\zeta^*} \leq 2  \sqrt{T}/(\sqrt{d}\ln(T) \ln(T/\delta))$.

Finally, we combine Eq.~\eqref{eq:thm-main-1}, Eq.~\eqref{eq:zeta-case1}, and Eq.~\eqref{eq:zeta-case2}, and have that
\begin{align*}
\sum_t \varpi_{\zeta_t, t}^{x_t} & \lesssim \delta d \sqrt{T} + d \sqrt{ T \log T \log (1/\delta)} \cdot \log \log (T/\delta) \nonumber\\
& \lesssim  d \sqrt{ T \log T \log (1/\delta)} \cdot \log \log (T/\delta) ,
\end{align*}
which is to be demonstrated.
\end{proof}

\end{document}